\documentclass{article}

\usepackage[preprint]{neurips_2019}

\usepackage[utf8]{inputenc} 
\usepackage[T1]{fontenc}    
\usepackage{hyperref}       
\usepackage{url}            
\usepackage{booktabs}       
\usepackage{amsfonts,amsmath,amsthm}       
\usepackage{nicefrac}       
\usepackage{microtype}      
\usepackage[capitalise]{cleveref}
\usepackage{xcolor}
\usepackage{graphicx}
\usepackage[italic,defaultmathsizes,symbolgreek]{mathastext}

\allowdisplaybreaks

\usepackage{selectp}

\title{Learning to Screen}

\author{%
	Alon Cohen\thanks{Technion---Israel Inst. of Technology and Google Research. {\tt aloncohen@technion.ac.il}}
	\And
	Avinatan Hassidim\thanks{Bar-Ilan University and Google Research. \tt{avinatanh@gmail.com}}
	\And
	Haim Kaplan\thanks{Tel-Aviv University and Google Research. {\tt haimk@tau.ac.il}}
	\And
	Yishay Mansour\thanks{Tel-Aviv University and Google Research. {\tt mansour.yishay@gmail.com}}
	\And
	Shay Moran\thanks{Princeton University. {\tt shaymoran1@gmail.com.} This work was done while the author was working at Google Research.}
}
\renewcommand{\H}{{\cal H}}
\newcommand{\R}{\mathbb{R}}
\newcommand{\wt}{\widetilde}
\newcommand{\wh}{\widehat}

\newcommand{\eps}{\epsilon}
\newcommand{\T}{{\cal T}}
\newcommand{\opt}{\mathsf{OPT}}
\DeclareMathOperator*{\Ex}{\mathsf{Ex}}

\newtheorem{theorem}{Theorem}
\newtheorem*{theorem*}{Theorem}
\newtheorem{lemma}[theorem]{Lemma}
\newtheorem*{lemma*}{Lemma}

\newtheorem{proposition}[theorem]{Proposition}

\theoremstyle{definition}
\newtheorem{definition}[theorem]{Definition}

\crefname{proposition}{proposition}{Propositions}

\begin{document}

\maketitle
\begin{abstract}
%
%
%
%
Imagine a large firm with multiple departments that plans a large recruitment. 
Candidates arrive one-by-one, and for each candidate
the firm decides, based on her data (CV, skills, experience, etc), 
whether to summon her for an interview.
The firm wants to recruit the best candidates while minimizing the number of interviews.
We model such scenarios as an assignment problem between items (candidates) and categories (departments):
the items arrive one-by-one in an online manner,
and upon processing each item the algorithm decides, 
based on its value and the categories it can be matched with,
whether to retain or discard it (this decision is irrevocable).
The goal is to retain as few items as possible while guaranteeing 
that the set of retained items contains an optimal matching. 
 
We consider two variants of this problem:
(i) in the first variant it is assumed that the $n$ items are drawn independently 
from an unknown distribution $D$.
(ii) In the second variant it is assumed that before the process starts, 
the algorithm has an access to a training set of $n$ items drawn independently 
from the same unknown distribution (e.g.\ data of candidates from previous recruitment seasons).
We give tight bounds on the minimum possible number of retained items in each of these variants.
These results demonstrate that one can retain exponentially less items in the second variant (with the training set).

Our algorithms and analysis utilize ideas and techniques from statistical learning theory and from discrete algorithms.

%
%
%
%

\end{abstract}

\section{Introduction}
Matching is the bread-and-butter of many real-life problems from the
fields of computer science, operations research, game theory, and economics.
Some examples include  job scheduling where we assign jobs to machines,
economic markets where we allocate products to buyers,  online advertising where we assign advertisers to ad slots,  assigning medical interns to
hospitals, and many more.
%
%
%

One particular example that motivates this work is the following example from labor markets. Imagine a firm that
is planning a large recruitment. Candidates arrive one-by-one and
the HR department immediately decides whether to summon them for an interview.
Moreover, the firm has multiple departments, each requiring
different skills and having a different target number of hires.
Different employees have different subsets of the required skills,
and thus fit only certain departments and with a certain
quality.
The firm's HR department, following the interviews,
decides which candidates to recruit and to which departments to assign them.
The HR department has to maximize the
total quality of the hired employees
such that each department gets its required number of hires with the required skills.
In addition, the HR uses data from the previous recruitment season in order to minimize the number of interviews while not compromising the quality of the solution.

We study the following formulation of the problem above. We receive $n$ items (candidates), 
where each item has a subset of $d$ properties (departments) denoted by $P_{1},\ldots,P_{d}$. We select $k$ items out of the $n$, subject to $d$ constraints of the form
\begin{center}
\textit{exactly $k_i$ of the selected items must satisfy a property $P_i$},
\end{center}
where $\sum_{i=1}^{d} k_{i} = k$ and we assume that $d \ll k \ll n$.
Furthermore, if item $c$ possesses property $P_i$, then it has a value $v_{i}(c)$ associated with this property.
Our goal is to compute a matching of maximum value that associates $k$ items to the $d$ properties subject to the constraints above.

We consider matching algorithms in the following online setting. The algorithms receive $n$ items online, drawn independently from $D$, and either reject or retain each item. Then, the algorithm utilizes the retained items and outputs an (approximately-)optimal feasible solution.
We present a naive greedy algorithm that returns the optimal solution with probability at least $1-\delta$ and retains $O(k \log (k/\delta))$ items in expectation. We prove that no other algorithm with the same guarantee can retain less items in expectation.

Thus, to further reduce the number of retained items, we add an initial preprocessing phase in which the algorithm learns an online policy from a \emph{training set}.
The training set is a single problem instance that consists of $n$ items drawn independently from the same unknown distribution $D$.
We address the statistical aspects of this problem and develop efficient learning algorithms.
In particular, we define a class of \emph{thresholds-policies}. Each thresholds-policy is a simple rule for deciding whether to retain an item. 
We present uniform convergence rates for both the number of items retained by a thresholds policy and the value of the resulting solution. We show that these quantities deviate from their expected value by order of $\sqrt{k}$ (rather than an easier $\sqrt{n}$ bound; recall that we assume $k \ll n$) which we prove using concentration inequalities and tools from VC-theory.
Using these concentration inequalities, we analyze an efficient online algorithm that returns the optimal offline solution with probability at least $1-\delta$, and retains a near-optimal $O(k\log\log (1/\delta))$ number of items in expectation (compare with the $O(k \log (k/\delta))$ number of retained items when no training set is given).

\paragraph{Related work.}
Our model is related to the online secretary problem in which one needs to select the best secretary in an online manner (see \citealp{ferguson1989solved}). Our setting differs from this classical model due to the two-stage process and the complex feasibility constraints. Nonetheless, we remark that there are few works on the secretary model that allow delayed selection (see \citealp{vardi2015returning,ezra2018prophets}) as well as matroid constraints \citep{babaioff2007matroids}. These works differ from ours in the way the decision is made, the feasibility constraints and the learning aspect of receiving a single problem instance as a training example.
{\citealp{Correa18Prophet} consider a distributional setting for the {\it single-choice prophet inequality} problem. 
Similarly to the setting considered here, they assume that the data is drawn independently from an unknown distribution
and that the algorithm has an access to a training-set sampled from the same distribution.
However, the objective is quite different from ours: the goal is to pick a stopping time $\tau$ such that the $\tau$'th sample
approximately maximizes the value among all samples (including those that were not seen yet).}

Another related line of work in algorithmic economics studies the statistical learnability of pricing schemes (see e.g.,~\citealp{Morgenstern15pseudo,Morgenstern16auctions,Hsu16prices,Balcan18theory}).
The main difference of these works from ours is that our training set consists of a single ``example'' (namely the set of items that are used for training),
and in their setting (as well as in most typical statistical learning settings)
the training set consists of many i.i.d examples.
This difference also affects the technical tools used for obtaining generalization bounds.
For example, some of our bounds exploit Talagrand's concentration inequality
rather than the more standard Chernoff/McDiarmid/Bernstein inequalities.
{We note that Talagrand's inequality and other advanced inequalities
were applied in machine learning in the context of learning combinatorial functions~\citep{Vondrak10concentration,Blum17opting}.
See also the survey by \cite{Boucheron03concentration} or the book by \cite{boucheron2013concentration}
for a more thorough review of concentration inequalities.}

Furthermore, there is a large body of work on online matching in which the vertices arrive in various models (see \citealp{mehta2013online,Gupta16experts}). We differ from this line of research, by allowing a two-stage algorithm, and requiring to output the optimal matching is the second stage.

\citet{celis2017ranking,celis2018multiwinner} studies similar problems of ranking and voting with fairness constraints.
In fact, the optimization problem that they consider allows more general constraints
and the value of a candidate is determined from votes/comparisons.
The main difference with our framework is that they do not consider a statistical setting
(i.e.\ there is no distribution over the items and no training set for preprocessing)
and focus mostly on approximation algorithms for the optimization problem.

%
%
%

\section{Model and Results} \label{sec:model}

Let $X$ be a domain of items, where each item $c\in X$ can possess any subset of $d$ properties
denoted by $P_{1},\ldots,P_{d}$ (we view $P_i\subseteq X$ as the set of items having property $P_i$).
Each item $c$ has a value $v_{i}(c) \in [0,1]$ associated with each property $P_i$ such that $c\in P_i$.


We are given a set $C\subseteq X$ of $n$ items as well as counts $k_1,\ldots k_d$ such that $\sum_{i=1}^{d} k_{i} = k$.
Our goal is to select exactly $k$ items in total, constrained on selecting exactly $k_{i}$ items with property $P_i$.
%
We assume that these constraints are {\em exclusive}, in the sense that each item in $C$ can be used to satisfy at most one of the constraints.
Formally, a feasible solution is a subset $S \subseteq C$, such that $\lvert S \rvert = k$ and  there is partition $S$  into $d$ disjoint subsets $S_1,\ldots, S_d$, such that
 $S_i \subseteq P_i$ and $\lvert S_{i} \rvert = k_{i}$.
We aim to compute a feasible subset $S$ that maximizes $\sum_{i=1}^{d} \sum_{c \in S_{i}} v_{i}(c)$.
%


Furthermore, we assume that $d \ll k \ll n$.
Namely, the number of constraints is much smaller than the number of items that we have to  select,
which is much smaller than the total number of items in~$C$.
In order to avoid feasibility issues we assume that there is a set $C_{\text{dummy}}$ that
contains $k$ dummy 0-value items with all the $d$ properties (we assume that the algorithm has always access to $C_{\text{dummy}}$ and do not view them as part of $C$).


\paragraph{Formulation as bipartite matching.} \label{sec:offline}
We first discuss the offline versions of these allocation problems. That is, we assume that
$C$ and the capacities $k_i$ are all given as an input before the algorithm starts.
We are interested in an algorithm for computing
an optimal set $S$. That is a set of items of maximum total value that satisfy the constraints.
%
This problem is equivalent to  a maximum matching problem in a bipartite graph~$(L,R,E,w)$ defined as follows.
\begin{itemize}
\item $L$ is the set of vertices in one side of the bipartite graph. It contains k vertices, where each constraint $i$ is represented by $k_i$ of these vertices.
\item $R$ is the set of vertices in the other side of the bipartite graph. It contains a vertex for each item $c\in C$ and for each dummy item $c'\in C_{\text{dummy}}$.
\item $E$ is the set of edges. Each vertex in $R$ is connected to each vertex of each of the constraints that it satisfies.
\item The weight $w(l,r)$ of edge $(l,r)\in E$ is $v_{l}(r)$: the value of item $r$ associated with property $P_{l}$. 
\end{itemize}
There is a natural correspondence between {\it saturated-matchings} in this graph, that is matchings in which every $l\in L$ is matched, and between {\it feasible solutions} (i.e.,\ solutions that satisfy the  constraints) to the 
allocation problem.
Thus,  a saturated-matching of maximum value corresponds to an optimal solution.
It is well know that the problem of finding such a  maximum weight bipartite matching can be solved in polynomial time (see e.g., \citealp{lawler2001combinatorial}).

\paragraph{Problem definition.}

In this work we consider the following online learning model. We assume that $n$ items are sequentially drawn i.i.d.\ from an unknown distribution $D$ over $X$. Upon receiving each item, we decide whether to retain it, or reject it irrevocably (the first stage of the algorithm). Thereafter, we select a feasible solution\footnote{In addition to the retained items, the algorithm has access to $C_{\text{dummy}}$, and therefore a feasible solution always exists.} consisting  \emph{only} of retained items (the second stage of the algorithm).
%
%
Most importantly, before accessing the online sequence and take irreversible online decisions of which items to reject, we have access a training set $C_{\text{train}}$ consisting of $n$ independent draws from $D$.

\subsection{Results}

\subsubsection{Oblivious online screening}
We begin by studying a greedy algorithm that does not require a training set. 
In the online phase, this algorithm acts greedily by keeping an item if it participates in the best solution thus far. Then, the algorithm computes an optimal matching among the retained items. The particular details of the algorithm are given in \cref{sec:online}. We have the following guarantee for this greedy algorithm proven in \cref{sec:online}.

\begin{theorem} \label{thm:greedyalg}
Let $\delta \in (0, 1)$. The greedy algorithm outputs the optimal solution with probability at least $1-\delta$ and retains $O(k \log (\min\{k/\delta, n/k\}))$ items in expectation.
\end{theorem}
 
As we shall see in the next section, learning from the training set allows one to retain exponentially less items than is implied by the theorem above.
It is then natural to ask to which extent is the training phase essential in order to accommodate such an improvement.
We answer this question in \Cref{sec:lowerboundramsey} by proving a lower bound on the number of retained items for {\it any} algorithm that does not use a training phase. This lower bound already applies in the simple setting where $d=1$:
here, each item consists only of a value~$v\in[0,1]$, and the goal of the algorithm is to retain as few items as possible
while guaranteeing with high probability that the top $k$ maximal values are retained.
\begin{theorem} \label{thm:ramsey}
Let $\delta \in (0,1)$.
For every algorithm $A$ which retains the maximal $k$ elements with
probability at least $1-\delta$, there exists a distribution $\mu$
such that the expected number of retained elements for input sequences $v_1\ldots v_n\sim \mu^n$
is at least $\Omega(k \log(\min\{k/\delta, n/k\})).$
\end{theorem}
Thus, the above theorem implies that $\Theta(k \log (n/k))$ can not be improved
even if we allow failure probability $\delta  = \Theta(k^{2}/n) $
(see \cref{thm:greedyalg}).

\subsubsection{Online screening with learning}
We now design online algorithms that, before the online screening process begins, use $C_{\text{train}}$ to learn a \emph{thresholds-policy} $T \in \T$ such that with high probability:
(i) the number of items that are retained in the online phase is small, and
(ii) there is a feasible solution consisting of $k$  retained items whose value is optimal (or close to optimal).
%
Thresholds-policies are studied in \cref{sec:thresholds} and are defined as follows.

\begin{definition}[Thresholds-policies] \label{def:threshpolicies}
A threshold-policy is parametrized by a vector $T=(t_1,\ldots,t_{d})$ of thresholds, where $t_i$ corresponds to
property $P_i$ for $1\leq i\leq d$.
The semantics of $T$ is as follows:
given a sample $C$ of $n$ items, each item $c \in C$ is retained if and only if there exists a property $P_{i}$ satisfied by $c$, such that its value $v_{i}(c)$ passes the threshold $t_{i}$.
More formally,  $c$ is retained if and only if $\exists i \in \{1,\ldots,d\}$ such that $c\in P_i$ and $v_{i}(c) \geq t_i$.
\end{definition}

Having proven uniform convergence results for thresholds-policies (see \Cref{sec:unifcov}), we show the following in \cref{sec:thresalg}.

\begin{theorem} \label{thm:informalthreshalg}
There exists an algorithm that learns a thresholds-policy $T$ from a single training sample $ C_{\text{train}} \sim D^n$,
such that after processing the (``real-time'') input sample $ C \sim D^n$ using $T$:
\begin{itemize}
\item It outputs an optimal solution with probability at least $1-\delta$.
\item The expected number of retained items in the first phase is $
    O\bigl(k(\log d + \log\log (n/k) +  \log\log (1/\delta)) \bigr).$
\end{itemize}
\end{theorem}
Thus, with the additional information given by the training set, the algorithm presented in \cref{thm:informalthreshalg} improves the number of retained items from $k\log (k/\delta)$ to $k \log \log (1/\delta)$. This demonstrates a significant improvement over \Cref{thm:greedyalg}.

Finally, in \cref{sec:lowerbound} we prove that the algorithm from \Cref{thm:informalthreshalg} is nearly-optimal in the sense
that it is impossible to significantly improve the number of retained items even if we allow the algorithm to fully know
the distribution over input items (so, in a sense, having an access to $n$ i.i.d samples from the distribution is the same as knowing it completely).

\begin{theorem} \label{thm:informallowerbound}
Consider the case where $k=d$ and $k_{1} = \cdots k_{d} = 1$. There exists a universe $X$ and a fixed distribution $D$ over $X$ such that for $C \sim D^{n}$ the following holds: any online learning algorithm (which possibly ``knows'' $D$) that retains a subset $S \subseteq C$ of items that contains an optimal solution with probability at least $1-\delta$ must satisfy
that $\Ex \bigl[\lvert S \rvert\bigr] = \Omega(k\log\log (1/\delta))$.
\end{theorem}

\section{The Greedy Online Algorithm}\label{sec:online}

A simple  way to collect a small set of items that contains the optimal solution is to select the~$k$ largest
items of each property. This set  clearly contains the optimal solution.
A simple argument, as in the proof of Lemma \ref{lem:size}, shows that this implementation of the first stage
keeps $O(kd\log (n/k))$ items on average.
In the following we present a \emph{greedy algorithm} that retains an average number of
$O(k\log (k/\delta))$ items in the first phase (for a parameter $\delta \in (0,1)$).



The greedy algorithm works as follows:
it ignores the first $\delta n / k$ items\footnote{We assume $\delta n / k$ is an integer without loss of generality.} and then starts processing the items one by one.
When we process the $i$'th item, $c_i$, the algorithm computes the optimal solution $M_i$ of
the first $i$ items (recall that we assume the algorithm has access to $C_{\text{dummy}}$, a large enough pool of zero valued items so there is always a feasible solution). The greedy algorithm retains $c_i$ if and only if $c_i$ participates in $M_i$.
We assume that $M_i$ is unique for every $i$ (we can achieve this with an arbitrary consistent tie breaking rule, say among matchings of the same
value we prefer the one that maximizes the sum of the indices of the matched items.).
Since the optimal solutions correspond to maximum-weighted bipartite-matchings between the items and the constraints, we have the following lemma.
\begin{lemma} \label{lem:opti-optn}
Suppose that the optimal solution, denoted by $M$, does not appear before round $\delta n / k$. Then it is a subset of the retained items.
\end{lemma}
\begin{proof}
Let $i \ge \delta n / k$. Consider an item $c$ matched by $M$ and assume
by contradiction that $c$ is not matched in~$M_i$.
Consider $Z=M \triangle M_i$ (we take the symmetric difference of $M$ and $M_i$ as sets of edges). Since
$M$ and $M_i$ do not necessarily
match the same items then the edges in $Z$ induce a collection of alternating paths and cycles where each path $L$
has an item matched by $M$ and not by $M_i$ at one end, and  an item matched by
$M_i$ and not by $M$ at the other hand. Except for its two ends, an alternating path
contains items  that are matched by both $M$ and $M_i$.
From the optimality and the uniqueness of $M$ follows that for each path the value of $M$ is larger than the value of
$M_i$.

Since $c$ is matched by $M$ and
not by $M_i$ there is a path $L$ in $Z$ that starts at $c$ and ends at some item that
is matched by $M_i$ and not by $M$.

It follows that all the items in $L$ are in $M_i$
and if we match them according to $M$ then the value that we gain from them increases.
This contradicts the optimality of $M_i$.

(Note that, in fact, there
are no cycles in $Z$, since they will imply that there are multiple optimal solutions, contradicting the uniqueness of $M_i$ and $M$.)
\end{proof}

Lemma \ref{lem:opti-optn} implies that, with high probability, if we collect all items that
are in the optimal solution of the subset of items that precedes them then
the set of items that we have at the end contains the optimal solution. Indeed, our algorithm fails if at least one of the items in the optimal solution $M$ is among the first $\delta n / k$ items. The probability that this occurs is at most $\delta$ via a union bound and the fact that the probability that any fixed item in $M$ is among the first $\delta n / k$ items is exactly $\delta / k$.

The next question is: how large is the subset of the items which we retain?
The next lemma answers this question  in an average sense.

\begin{lemma} \label{lem:size}
Assume that at the first stage the algorithm receives the items in a random order.
Then the expected number of items that the first stage keeps is
$O \bigl(k\log \min \big\{\tfrac{n}{k}, \tfrac{k}{\delta} \big\} \bigr)$.
\end{lemma}
\begin{proof}
Let $i \ge \delta n / k$ and denote $X_i$ as an indicator that is one if and only if the $i$'th item belongs to $M_i$.
Condition the probability space on the {\em set} $L_i$ of the first $i$ items (but not on their order).
Each element of $L_i$ is equally likely to arrive last. So since $\lvert M_i\rvert \leq k$,
then the probability that the element arriving last in $L_i$ is in $\opt_i$ is at most $k/i$ if $k< i$
or at most $1$ otherwise.
It follows that $E[X_i\mid L_i] \leq \min \bigl\{ \tfrac{k}{i},1 \bigr\}$.
Since this holds for any $L_i$,  it also holds unconditionally as well.
Therefore, if $\delta n / k < k$ then by the fact that $\sum_{i=k+1}^{n} \tfrac{1}{i} \le \log \tfrac{n}{k}$, the expected number of retained items is
\[
	k - \frac{\delta n}{k} + \sum_{i = k + 1}^{n} \frac{k}{i} = O \bigg(k \log \frac{n}{k}\bigg).
\]
Similarly, if $\delta n / k \ge k$ then the expected number of retained items is
\[
	\sum_{i = \delta n / k + 1}^{n} \frac{k}{i} = O \bigg(k \log \frac{k}{\delta} \bigg). \qedhere
\]
\end{proof}


\section{Thresholds-policies}\label{sec:thresholds}

We next discuss a framework to design algorithms that exploit the training set to learn
policies that are applied in the first phase of the matching process.
We would like to frame this in standard ML formalism by
phrasing this problem as learning a class $\H$ of policies such that:
\begin{itemize}
\item {\bf $\H$ is not too small:}
The policies in $\H$ should yield solutions with high values (optimal, or near-optimal).
\item {\bf $\H$ is not too large:}
$\H$ should satisfy some uniform convergence properties;
i.e.\ the performance of each policy in $\H$ on the training set
is close, with high probability, to its expected real-time performance
on the sampled items during the online selection process.
\end{itemize}

Indeed, as we now show these demands are met by the class $\T$ of thresholds policies (\cref{def:threshpolicies}).
We first show that the class of thresholds-policies contains an optimal policy, and in the sequel we show that it satisfies attractive uniform convergence properties.

\paragraph{An assumption (values are unique).}
We assume that for each constraint $P_i$, the marginal distribution over the value of $c \sim D$ conditioned on $c\in P_i$ is atomless; namely $\Pr_{c\sim D}[v(c) = v \mid c\in P_i]=0$ for every $v\in [0,1]$.
This assumption can be removed by adding artificial tie-breaking rules, but making it will simplify some of the technical statements.

\begin{theorem}[There is a thresholds policy that retains an optimal solution] \label{thm:thresopt}
For any set of items $C$, there exists a thresholds vector $T\in \T$
that retains exactly $k$ items that form an optimal solution for $C$.
\end{theorem}
\begin{proof}
Let $S$ denote the set of $k$ items in an optimal solution for $C$,
and let $S_i\subseteq S\cap P_i$ be the subset of $M$ that is assigned to the constraint $P_i$.
Define $t_i = \min_{c\in S_i}v_{i}(c)$, for $i\geq 1$, 
Clearly, $T$ retains all the items in $S$.
Assume towards contradiction that $T$ retains an item $c_j\notin S$,
and assume that $P_i$ is a constraint such that $c_j\in P_i$ and $v_{i}(c_j) \geq  t_i$.
Since by our assumption on $D$ all the values $v_{i}(c_j)$ are distinct it follows that $v_{i}(c_j) > t_i$.
Thus, we can modify $S$ by replacing $c_j$ with the item of minimum value in~$S_i$
and increase the total value. This contradicts the optimality of $S$.
\end{proof}

We next establish generalization bounds for the class of thresholds-policies.

\subsection{Uniform convergence of the number of retained items}\label{sec:unifcov}

For a sample $C \sim D^{n}$ and a thresholds-policy $T \in \T$, we denote by $R_{i}^{T}(C) = \{c : c\in P_i \text{ and } v_{i}(c) \ge t_{i}\}$ the set of items that are retained by the threshold $t_{i}$, and we denote its expected size by $\rho_{i}^T=\Ex_{C\sim D^n} \bigl[ \lvert R_{i}^T(C) \rvert \bigr]$. Similarly we denote by $R^{T}(C) = \cup_{i} R_{i}^{T}(C)$ the items retained by $T$, and by $\rho^{T}$ its expectation.
We prove that the sizes of $R_{i}^{T}(C)$ and $R^{T}(C)$ are concentrated around their expectations uniformly for all thresholds policies.

The following theorems establish uniform convergence results for the number of retained items. Namely, with high probability we have $R^{ T}_{i} \approx \rho^{ T}_{i}$, $ R^{ T} \approx \rho^{T}$
simultaneously for all $ T\in \T$ and $i\leq d$.

\begin{theorem}[Uniform convergence of the number of retained items] \label{thm:ucretained}
With probability at least $1-\delta$ over $C\sim D^n$, the following holds for all policies $ T \in \T$ simultaneously:
\begin{enumerate}
\item If $\rho^{ T} \ge k$, then $(1-\epsilon) \rho^{ T} \le \lvert R^{ T}(C) \rvert \le (1+\epsilon) \rho^{ T}$~,~\mbox{and}
\item if $\rho^{T} < k$, then $\rho^{ T} - \epsilon k \le \lvert R^{ T}(C) \rvert \le \rho^{ T} + \epsilon k$~,
\end{enumerate}
where
\[
\epsilon = O\left( \sqrt{\frac{d \log (d) \log (n / k) + \log(1/\delta) }{k}} \right).
\]
\end{theorem}

\begin{theorem}[Uniform convergence of the number of retained items per constraint] \label{thm:ucretainedper}
With probability at least $1-\delta$ over $C\sim D^n$, the following holds for all policies $ T \in \T$ and all $i\leq d+1$ simultaneously:
\begin{enumerate}
\item If $\rho^{ T}_{i} \ge k$, then $(1-\epsilon) \rho^{ T}_{i} \le \lvert R^{ T}_{i}(C) \rvert \le (1+\epsilon) \rho^{ T}_{i}$~,~\mbox{and}
\item if $\rho^{ T}_{i} < k$, then $\rho^{ T}_{i} - \epsilon k \le \lvert R^{ T}_{i}(C) \rvert \le \rho^{ T}_{i} + \epsilon k$~,
\end{enumerate}
where
\[
\epsilon = O\left( \sqrt{\frac{\log (d) \log (n / k) + \log(1/\delta) }{k}} \right)~.
\]
\end{theorem}



The proofs of \Cref{thm:ucretained,thm:ucretainedper} are based on standard VC-based
uniform convergence results, and technically the proof boils down to bounding the VC-dimension of the families
\[{\cal R} = \{R^{ T} :  T\in\T\}~~\text{ and } ~~{\cal Q} = \{R^{ T}_{i} :  T\in\T,\ i\leq d\}.\]
Indeed, in \cref{sec:generalization} we prove the following.

\begin{lemma}
\label{lem:vcupperbound}
$\mbox{VC}({\cal R}) = O(d \log d)~.$
\end{lemma}

\begin{lemma}\label{lem:vc2}
$\mbox{VC}({\cal Q}) = O(\log d)~.$
\end{lemma}

Using \cref{lem:vcupperbound,lem:vc2}, we can now apply standard uniform convergence results from VC-theory to derive \Cref{thm:ucretained,thm:ucretainedper}.

\begin{definition}[Relative $(p,\epsilon)$-approximation; \citealp{har2011relative}]\label{def:relative}
Let ${\cal F}$ be a family of subsets over a domain $X$, and let $\mu$ be a distribution on $X$.
$Z\subseteq X$ is a $(p,\epsilon)$-approximation for ${\cal F}$ if for each $f\in F$ we have,
\begin{enumerate}
\item If $\mu(f) \ge p$, then $(1-\epsilon)\mu(f) \le \wh \mu(f) \le (1+\epsilon) \mu(f)$,
\item If $\mu(f) < p$, then $\mu(f)- \epsilon p \le \wh \mu(f) \le \mu(f) + \epsilon p$,
\end{enumerate}
where $\wh \mu(f) = {\lvert Z\cap F\rvert}/{\lvert Z\rvert}$ is the (``empirical'') measure of $f$ with respect to $Z$.
\end{definition}

The proof of \cref{thm:ucretained,thm:ucretainedper} now follows by plugging $p = k / n$ in \citet[Theorem 2.11]{har2011relative},
which we state in the next proposition.
\begin{proposition}[\citealp{har2011relative}]
Let ${\cal F}$ and $\mu$ like in \Cref{def:relative}.
Suppose ${\cal F}$ has VC dimension~$m$.
Then, with provability at least $1-\delta$, a random sample of size
\[
\Omega \left( \frac{m \log(1/p) + \log(1/\delta)}{\epsilon^{2} p} \right)
\]
is a relative $(p,\epsilon)$-approximation for ${\cal F}$.
\end{proposition}

\subsection{Uniform convergence of values}

We now prove a concentration result for the value of an optimal solution among the retained items. 
Unlike the number of retained items, the value of an optimal solution corresponds to a more complex random variable, and analyzing the concentration of its empirical estimate requires more advanced techniques.

We denote by $V^{T}(C)$ the value of the optimal solution among the items retained by the thresholds-policy $T$, and we denote its expectation by $\nu^T=\Ex_{C\sim D^n} \big[V^T(C) \big]$. We show that $V^{T}(C)$ is concentrated uniformly for all thresholds policies.
%

\begin{theorem}[Uniform convergence of values] \label{thm:ucvalues}
With probability at least $1-\delta$ over $C\sim D^n$, the following holds for all policies~$T \in \T$ simultaneously:
\[\big\lvert \nu^{ T} -  V^{ T}(C) \big\rvert \leq \eps k, \;\; \text{where} \;\; \eps = O\Biggl(\sqrt{\frac{d\log {k} + \log(1/\delta)}{k}}\Biggr).\]
\end{theorem}

Note that unlike most uniform convergence results that guarantee simultaneous convergence of empirical averages to expectations, here $V^{T}(C)$ is not an average of the $n$ samples, but rather a more complicated function of them.
We also note that a bound of $\wt{O}(\sqrt{n})$ (rather than $\wt{O}(\sqrt{k})$)  on the additive deviation of $V^T(C)$ from
its expectation can be derived using the  McDiarmid's inequality~\citep{Mcdiarmid1989}.
However, this bound is meaningless when $\sqrt{n} > k$  (because $k$ upper bounds the value of the optimal solution).
We use Talagrand's concentration inequality~\citep{Talagrand1995} to
 derive the $O(\sqrt{k})$ upper bound on the additive deviation. Talagrand's concentration inequality allows us to utilize the fact that an optimal solution uses only $k \ll n$ items, and therefore replacing an item that does not participate in the solution does not affect its value.

To prove the theorem we need the following concentration inequality for the value of the optimal selection in hindsight. Note that by  \cref{thm:thresopt} this value equals to $V^{ T}(C)$ for some $T$.

\begin{lemma}\label{lem:single}
Let $\opt(C)$ denote the value of the optimal solution for a sample $C$. We have that
\[
\Pr_{ C\sim D^n}\bigl[\lvert \opt(C) - \Ex[\opt(C)] \rvert \geq \alpha\bigr]\leq 2\exp(-{\alpha^2/2k}).
\]
\end{lemma}

So, for example, it happens that $\lvert \opt(C) - \Ex[\opt(C)] \rvert \le \sqrt{2k \log(2/\delta)}$ with probability at least $1-\delta$.

To prove this lemma we use
the following version of Talagrand's inequality (that appears for example in lecture notes by~\citet{van2014probability}).

\begin{proposition}[Talagrand's Concentration Inequality]
Let $f:\R^n \mapsto \R$ be a function, and suppose that there exist $g_1,\ldots,g_n : \R^{n} \mapsto \R$ such that for any $x,y\in \R^n$
\begin{equation}
\label{eq:talagrandassumption}
f(x) - f(y) \le \sum_{i=1}^n g_i(x) 1_{[x_i \neq y_i]}.
\end{equation}
Then, for independent random variables $X = (X_1,\ldots,X_n)$ we have
\[
\Pr \left[ \lvert f(X) - \Ex[f(X)] \rvert > \alpha \right] \le 2 \exp \left(-\frac{\alpha^2}{2 \sup_x \sum_{i=1}^n g_i^2(x)} \right).
\]
\end{proposition}

\begin{proof}[Proof of \cref{lem:single}]
We apply Talagrand's concentration inequality to the random variable $\opt(C)$.
Our $X_i$'s are the items $c_{1},\ldots,c_{n}$ in the order that they are given.
We  show that \cref{eq:talagrandassumption} holds for $g_i(C) = 1_{[c_i \in S]}$ where $S=S(C)$ is a fixed optimal solution for~$C$ (we use
some arbitrary tie breaking among optimal solutions).
We then have, $\sum_{i=1}^n g_i^2(C) = \lvert S \rvert = k$, thus completing the proof.

Now, let $C$, $ C'$ be two samples of $n$ items.
Recall that we need to show that
\[
\opt(C) - \opt(C') \le \sum_{i=1}^n g_i( C) 1_{[ c_i \neq  c'_i]}~.
\]
We use $S$ to construct a solution $S'$ for $ C'$ as follows.
Let $S_j \subseteq S$ the subset of $S$ matched to $P_{j}$. For each $i$,  if
$c_{i} \in S_{j}$ for some $j$, and
$ c_i =  c_i'$, then we add $i$ to $S'_j$.
Otherwise, we add a dummy item from $C_\text{dummy}'$ to $S'_j$ (with value zero).
Let $V(S')$ denote the value of $S'$.
Note that the difference between the values of $S$ and $S'$ is the total value of all items $i \in S$ such that $ c_i \neq  c'_i$. Since the item values are bounded in $[0,1]$ we get that
\[
    \opt( C) - V(S')
    =
    \sum_{j=1}^{d} \sum_{c_i \in S_{j}} v_{j}( c_i) 1_{[ c_i \neq  c'_i]}
    \le
    \sum_{j=1}^{d} \sum_{c_i \in S_{j}} 1_{[ c_i \neq  c'_i]}
    =
    \sum_{i=1}^n g_i( C) 1_{[c_i \neq  c'_i]}~.
\]
The proof is complete by noticing that $\opt( C') \ge V(S')$.
\end{proof}

We also require the following construction of a bracketing of $\T$ which is formally presented in \cref{sec:generalization}.

\begin{lemma} \label{lem:packing}
There exists a collection of ${\cal N}$ thresholds-policies such that $\lvert {\cal N}\rvert \leq k^{O(d)}$,
and for every thresholds-policy $ T\in \T$ there are ${ T}^+,{ T}^-\in {\cal N}$ such that
\begin{enumerate}
\item
\label{item:1}
 $V^{ T^-}(C)\leq V^{ T}(C)\leq V^{ T^+}(C)$ for every sample of items $C$;
note that by taking expectations this implies that  $ \nu^{ T^-}\leq \nu^{ T}\leq  \nu^{ T^+}$, and
\item
\label{item:2}
$\nu^{ T^+} - \nu^{ T^-} \leq {10}$.
\end{enumerate}
\end{lemma}

\begin{proof}[Proof of \cref{thm:ucvalues}]
The items in $C$ that are retained by $T$ are independent samples from a distribution $D'$ that is sampled as follows:
(i) sample $c\sim D$, and (ii) if $c$ is retained by $ T$ then keep it, and otherwise discard it.
This means that $v^T(C)$ is in fact the optimal solution of $C$ with respect to $D'$.
Since \Cref{lem:single} applies to \emph{every} distribution $D$ we can apply it to $D'$ and get that
for any fixed $T\in {\cal T}$
\[
\Pr_{ C\sim D^n}\bigl[\lvert \nu^{ T} -  V^{ T}(C) \rvert \geq \alpha\bigr]\leq 2\exp(-{\alpha^2/2k})\ .
\]

Now, by the union bound for  ${\cal N}$ be as in \cref{lem:packing} we get  that the probability that there is $ T\in {\cal N}$ such that $\lvert \nu^{ T} - V^{T}(C) \rvert \geq \alpha$ is at most
$
    \lvert{\cal N}\rvert \cdot  2 \exp(-\alpha^2/2k)
$.
Thus, since $\lvert {\cal N}\rvert\leq k^{O(d)}$, it follows that
with probability at least $1-\delta$,
\begin{equation}\label{eq:net}
(\forall T\in {\cal N}): ~ \lvert \nu^{ T} - V^{ T}(C) \rvert \leq O\Bigl(\sqrt{k\bigl(d\log {k} + \log(1/\delta)\bigr)}\Bigr)\ .
\end{equation}

We now show why uniform convergence for ${\cal N}$ implies uniform convergence for $\T$.
Combining Lemma \ref{lem:packing} with \Cref{eq:net} we get that with probability at least $1-\delta$,
every~$T\in \T$ satisfies:
\begin{align}
\lvert  \nu^{ T} - V^{ T}(C) \rvert
    &\leq
\max\{ \lvert  \nu^{ T^+} -  V^{ T^-}(C) \rvert, \lvert  \nu^{ T^-} - V^{ T^+}(C) \rvert \} \nonumber
\tag{by \cref{item:1} of \cref{lem:packing}}\\
    &\leq
\max\{ \lvert  \nu^{ T^-} - V^{ T^-}(C) \rvert, \lvert  \nu^{ T^+} -  V^{ T^+}(C) \rvert \} + {10}\nonumber
\tag{by \cref{item:2} of \cref{lem:packing}}\\
    &\leq {10} + O\Bigl(\sqrt{k\bigl(d\log {k} + \log(1/\delta)\bigr)}\Bigr). \tag{by \cref{eq:net}} \nonumber
\end{align}
Here the first inequality follows from Item 1 by noticing that if $[a,b]$, $[c,d]$ are intervals on the real line and $x \in [a,b]$, $y \in [c,d]$ then $|x-y| \le \max \{|b-c|,|d-a|\}$, and plugging in $x=\nu^{ T}, y= V^T(C), a=\nu^{ T^-},b=\nu^{ T^+},c=V^{T^-}(C), d= V^{T^+}(C)$.

This finishes the proof, by setting $\eps$ such that
$
\eps\cdot k = O\bigl(\sqrt{k (d\log {k} + \log(1/\delta) )} \bigr)
$.
\end{proof}

\section{Algorithms based on learning thresholds-policies}\label{sec:thresalg}

We next exemplify how one can use the above properties of thresholds-policies
to design algorithms.
A natural algorithm would be to use the training set to learn a threshold-policy $ T$ that retains an optimal solution with $k$ items from the training set as specified  in \Cref{thm:thresopt}, and then use this online policy to retain a subset of the $n$ items in the first phase.
\Cref{thm:ucretained}  and \Cref{thm:ucvalues} imply that with  probability $1-\delta$, the number of retained items is at most $m=k +   O\bigl(\sqrt{kd \log (d) \log(n/k) + k \log(1/\delta)} \bigr)$ and  that the value of the resulting solution is at least~$\opt - O\bigl(\sqrt{kd \log k + k \log(1/\delta)}\bigr)$.

We can improve this algorithm by combining it with the greedy algorithm
of \Cref{thm:greedyalg} described in \cref{sec:online}.
During the first phase, we retain an item $c$ only if (i) $c$ is retained by $T$, and
(ii) $c$ participates in the optimal solution among the items that were retained thus far.
\cref{thm:greedyalg} then implies that out of these $m$ items greedy keeps a subset of
\[
    O \left(k\log \frac{m}{k} \right) = O \left(k \left(\log\log \left(\frac{n}{k} \right) +\log\log\left(\frac{1}{\delta} \right) \right) \right).
\]
items in expectation that still contains a solution of value at least $\opt - O(\sqrt{kd \log k + k \log(1/\delta)})$.

We can further improve the value of the solution
and guarantee that it will be optimal (with respect to all $n$ items) with probability $1-\delta$.
This is based on the observation that if the set of retained items contains the top $k$ items of each property $P_i$ then it also contains an optimal solution.
Thus, we can compute a thresholds-policy $T$ that retains the top $k + O(\sqrt{k\log (d) \log (n/k) + k \log(1/\delta)} )$ items of each property
from the training set (if the training set does not have this many items with some property then set the corresponding threshold to $0$).
Then, it follows from  \Cref{thm:ucretainedper}, that with probability $1-\delta$, $T$ will  retain the top~$k$ items of each property in the first online phase and therefore will retain an optimal solution.
Now, \Cref{thm:ucretainedper} implies that with probability $1-\delta$ the total number of items that are retained by $ T$
in real-time is at most $m=dk + O(d\sqrt{k\log (d) \log (n/k) + k \log(1/\delta)} )$.
By filtering the retained elements  with the greedy algorithm of \Cref{thm:greedyalg} as before it follows that the total number of retained items is at most
\[k + k\log\Bigl(\frac{m}{k}\Bigr) =  O \left(k\left( \log d +  \log\log \left( \frac{n}{k} \right) +  \log\log\left(\frac{1}{\delta} \right) \right) \right)\]
with probably $1-\delta$.
This proves \cref{thm:informalthreshalg}.

\section{Lower Bounds} 

\subsection{Necessity of the training phase}\label{sec:lowerboundramsey}

Let $n\in \mathbb{N}$ (sample size) and $\delta\in [0,1]$ (confidence parameter).
In this section we focus on the case where there is no training phase and $d=1$.
Thus, we consider algorithms which get as an input a sequence $v_1,\ldots v_n\in [0,1]$
in an online manner (one after the other). In step $m$ the algorithm needs to decide
whether to retain $v_m$ or to discard it (this decision may depend on the prefix $v_1\ldots,v_{m}$).
The algorithm is not allowed to discard a sample after it has been retained.

The following property captures the utility of the algorithm:
{\it for every distribution $\mu$ over $[0,1]$, if $v_1,\ldots, v_n$ are sampled i.i.d from $\mu$,
then with probability at least $1-\delta$, the algorithm retains $v_{j_{1}},\ldots,v_{j_{k}}$ that are the largest $k$ elements in $v_{1},\ldots, v_{n}$.}
The goal is to achieve this while minimizing the number of retained items in expectation.

\begin{theorem*}[\Cref{thm:ramsey} restatement]
Let $\delta \in (0,1)$.
For every algorithm $A$ which retains the maximal $k$ elements with
probability at least $1-\delta$, there exists a distribution $\mu$
such that the expected number of retained elements for input sequences $v_1\ldots v_n\sim \mu^n$
is at least $\Omega(k \log(\min \{n/k, {k}/{\delta}\}))$.
\end{theorem*}

We remind that the bound is tight for the greedy algorithm (\cref{thm:greedyalg}).

\begin{proof}

Following~\cite[Corollary 3.4]{moran1985applications}, we may assume that $A$ accesses its input only 
using comparisons. More precisely: call two sequences $v_1,\ldots, v_m$
and $u_1,\ldots, u_m$ {\it order-equivalent} if~$v_i \leq v_j \iff u_i\leq u_j$
for all $i,j\leq m$, and call the equivalence class of $v_1\ldots v_m$ its {\it order-type}. 
Note that if $v_1, \ldots, v_m$ are distinct, then their order-type is naturally identified with a permutation $\sigma\in \mathbb{S}_{m}$.
Call an algorithm $A$ {\it order-invariant} if for every $m\leq n$,
the decision\footnote{When $A$ is randomized then the value of $\Pr[A \text{ retains } v_k]$ depends only on the order-type of $v_1\ldots v_m$.} of $A$ whether to retain $v_m$ depends only on the order-type of $v_1,\ldots, v_m$
(equivalently, $A$ accesses the input only using comparisons).

By~\cite{moran1985applications} it follows that for every algorithm $A$ 
there is an infinite $W\subseteq [0,1]$ such that $A$ is order-invariant
when restricted to input sequences $v_1,\ldots, v_n\in W$.
For the remainder of the proof we fix such an infinite set $W$ 
and focus only on inputs from $W$.

Set $\mu$ to be a uniform distribution over a sufficiently large subset of $W$
so that $v_1\ldots v_n$ are distinct with probability $1-1/n$. Let $\opt(S)$ denote the top $k$ elements in $S$.
Let $T_m$ be the set of all sequences $v_1,\ldots,v_m,\ldots,v_n \in W^n$
such that $v_m \in \opt(\{v_1,\ldots ,v_{m}\})$, and let
$p_k$ denote the probability that $A$ retains $v_m$ conditioned on the input being from $T_m$.
Let $T'_m\subseteq T_m$ denote the set of all sequences $v_1,\ldots,v_m,\ldots,v_n$ such that
$v_m \in \opt(\{v_{1},\ldots,v_{n}\})$ (i.e., $v_{m}$ is part of the optimal solution).
The proof hinges on the following lemma:
\begin{lemma}
Since $A$ is order based, for every $m\leq n$, 
$p_m$ is also the probability that $A$ retains $v_m$ conditioned on the input being from
$T'_m$.
\end{lemma}
\begin{proof} 
The decision of $A$ whether to retain $v_m$ depends only on the order-type of~$v_1,\ldots, v_m$.
For each $\sigma\in \mathbb{S}_m$, let $E(\sigma)$ denote the event that the order type of $v_1\ldots v_m$ is $\sigma$.
Thus,
\[
p_m= \Pr[A \text{ retains $v_m$} \mid T_m] = \sum_{\sigma\in \mathbb{S}_m}\Pr[E(\sigma) \mid T_m]\cdot \Pr[A \text{ retains $v_m$ }\mid E(\sigma)],
\]
and similarly
\[
\Pr[A \text{ retains $v_m$} \mid T'_m] = \sum_{\sigma\in \mathbb{S}_m}\Pr[E(\sigma) \mid T'_m]\cdot \Pr[A \text{ retains $v_m$ }\mid E(\sigma)].
\]
Next, observe that for each order-type $\sigma\in \mathbb{S}_m$:
\begin{align*}\Pr[E(\sigma) \mid T_m ] &=  \Pr [E(\sigma) \mid T'_m ]= 
			\begin{cases}
			 \frac{1}{m!}, &m \le k \\
			 \frac{1}{k(m-1)!}, &v_m \in \opt(\{v_{1},\ldots,v_{m}\}), m > k \\
			 0,	         &\text{otherwise}.
			 \end{cases} \qedhere
\end{align*}
\end{proof}

With the above lemma in hand, we can finish the proof.
For the remainder of the argument, 
we condition the probability space on the event that all elements in the sequence $v_1, \ldots, v_n$
are distinct and show that conditioned on this event,
$A$ retains at least $t=\Omega(\log(1/\delta))$ elements in expectation.
Note that this will conclude the proof since by the choice of $\mu$ 
this event occurs with probability $\geq 1-1/n$,
which implies that -- unconditionally -- $A$ retains at least $t - n\cdot(1/n)= t-1 =\Omega(\log(1/\delta))$
elements in expectation.

For each $m\leq n$, $v_m$ is among the top $k$ elements with probability $\min\{1, k/m\}$, 
in which case it is retained with probability $p_m$. So $A$ retains at least 
\[
	\sum_{m=1}^n \min\bigg\{1, \frac{k}{m}\bigg\} \cdot p_m
\] 
elements in expectation. By the above lemma, the probability that $A$ discards the maximum is
\[
\sum_{m=1}^n \frac{k}{n} \cdot (1-p_m),
\] 
which by assumption is smaller than $\delta$.
So we obtain  that $\sum p_m \ge n(1-\delta / k)$.
Thus, to minimize $\sum_{m=1}^n \min\{1,k/m\} \cdot p_m$ subject to 
the constraint that $\sum p_m \ge n(1-\delta / k)$
we make the last $n(1-\delta/k)$ $p_m$'s equal to $1$ and the rest $0$. 
This gives the desired lower bound.
\end{proof}

\subsection{The algorithm from \Cref{thm:informalthreshalg} is optimal}\label{sec:lowerbound}
In the previous section we have presented an algorithm that with probability at least $1-\delta$ outputs an optimal solution
while retaining at most $O(k(\log\log n +  \log d + \log\log(1/\delta)))$ items in expectation during the first phase.

We now present a proof of \cref{thm:informallowerbound}.
We start with the following lemma that shows the dependence on $\delta$ cannot be improved in general, even for $k=1$,
when there are no constraints, and the distribution over the items is known to the algorithm
(so there is no need to train it on a sample from the distribution):
\begin{lemma}\label{lem:lb}
Let $v_1,\ldots,v_n\in[0,1]$ be drawn uniformly and independently, let $e^{-n/2} < \delta < {1/10}$
and let $A$ be an algorithm that retains the maximal value among the $v_i$'s with probability at least $1-\delta$.
Then, 
\[\Ex \bigl[\lvert S\rvert \bigr] = \Omega\left(\log\log\left( \frac{1}{\delta} \right) \right), \]
where $S$ is the set of values retained by the algorithm.
\end{lemma}
%
Thus, it follows that for $\delta= \mathsf{poly}(1/n)$ and $k,d = O(1)$
the bound in \cref{thm:informalthreshalg} is tight.

\begin{proof}
Define ${\alpha = \frac{\ln(1/\delta)}{2n}\in(1/n,1/4)}$. Let $E_t$ denote the event that $v_t\geq 1-\alpha$ and is the largest among $v_1,\ldots,v_t$.
We have that
\begin{equation} \label{eq:lower-bound}
\Ex[\lvert S\rvert] \ge \sum_t \Pr[\text{$v_t$ is picked and $E_t$}] =  \sum_t \left( \Pr[E_t]  - \Pr[\text{$v_t$ is rejected and $E_t$}] \right) \ .
\end{equation}
We show that since $A$ errs with probability at most $\delta$ then $\sum_t \Pr[\text{$E_t$  and $v_t$ is rejected}]$ is small.

\begin{align*}
\delta\geq\Pr[~\text{$A$ rejects $v_{max}$}] &\geq  \sum_{t}\Pr[~\text{$A$ rejects $v_t$ and $E_t$ and $v_t=v_{max}$}]\\
                            &= \sum_t\Pr[v_t=v_{max}~~\vert ~~\text{$A$ rejects $v_t$ and $E_t$}]\cdot\Pr[\text{$A$ rejects $v_t$ and $E_t$}]\\
                            &\geq \sum_t\Pr[v_i\leq 1-\alpha \text{ for all }~i> t~~\vert ~~\text{$A$ rejects $v_t$ and $E_t$}]\cdot\Pr[\text{$A$ rejects $v_t$ and $E_t$}]\\
                            &=\sum_t\Pr[v_i\leq 1- \alpha \text{ for all }~i> t]\cdot\Pr[\text{$A$ rejects $v_t$ and $E_t$}]\\
                            &\geq \sum_t (1-\alpha)^{n-t}\cdot\Pr[\text{$A$ rejects $v_t$ and $E_t$}]\\
                            &\geq (1-\alpha)^n\sum_t \Pr[\text{$A$ rejects $v_t$ and $E_t$}].
\end{align*}
The crucial part of the above derivation is in third line.
It replaces the event ``$v_t=v_{max}$'' by the event ``$v_i\leq 1-\alpha \text{ for all }~i> t$''
(which is contained in the event ``$v_t=v_{max}$'' under the above conditioning).
The gain is that the events ``$v_i\leq 1-\alpha \text{ for all }~i> t$''
and ``$\text{$A$ rejects $v_t$ and $E_t$}$'' are independent
(the first depends only on $v_i$ for $i>t$ and the latter on $v_i$ for $i\leq t$).
This justifies the ``$=$'' in the fourth line.

Rearranging, we have
$\sum_t \Pr[\text{$A$ rejects $v_t$ and $E_t$}]  \le \frac{\delta}{(1-\alpha)^n}$.
Substituting this bound in \cref{eq:lower-bound},
\begin{align*}
\Ex[\lvert S\rvert] & \ge  \sum_t \Pr[\text{$v_t$ is picked and $E_t$}] \\
&= \sum_t  \left( \Pr[E_t]  - \Pr[\text{$v_t$ is rejected and $E_t$}] \right) \\
                  &= \sum_t \Pr[E_t]  - \frac{\delta}{(1-\alpha)^n}\\
                   &\geq \frac{1}{4}\ln(\alpha n) - \delta\cdot \exp(2\alpha n) \tag{explained below}\\
                   &= \frac{1}{4}\ln\biggl(\frac{\ln(1/\delta)}{2}\biggr) - \delta\exp(\ln(1/\delta))  \tag{by the definition of $\alpha$}\\
                   &= \frac{1}{4}\ln\ln(1/\delta) - \frac{1}{4}\ln 2 - 1 = \Omega(\log\log(1/\delta)),
\end{align*}
which is what we needed to prove.
The last inequality follows because
\begin{itemize}
\item[(i)] $\sum_t \Pr[E_t] \geq\frac{1}{4}\ln(\alpha n)$ (as is explained next), and
\item[(ii)] $1-\alpha \geq\exp(-2\alpha)$ for every $\alpha\in [0,\frac{1}{4}]$ (which can be verified using basic analysis).
\end{itemize}

To see (i), note that
\[ \sum_t \Pr[E_t] = \Ex\biggl[\sum_t 1_{E_t}\biggr] .\]
Let $z = \lvert\{t : v_t\geq 1-\alpha \}\rvert$.
Since the $v_i$'s are uniform in $[0,1]$ then by the same argument as in the proof of
 \Cref{lem:size} we get that
\[\Ex\left[\sum_t 1_{E_t} \mid z\right] = \sum_{i=1}^{z} \frac{1}{i}\geq \int_1^{z+1} \frac{1}{x} =\ln(z+1) ,\]
and therefore
\[\Ex \left[\sum_t 1_{E_t} \right]  = \Ex_z\Ex\left[\sum_t 1_{E_t} \mid z\right] \ge \Ex_z\left[\ln(z+1)\right] .\]
Let $Z\sim \text{Bin}(n,\alpha)$, and therefore we need to lower bound $\Ex[\ln(Z+1)]$ for $Z\sim \text{Bin}(n,\alpha)$.
To this end, we use the assumption that $\alpha > 1/n$,
and therefore $\Pr[Z \geq \alpha\cdot n]\geq 1/4$ (see \citealp{Greenberg13binomial} for a proof of this basic fact).
In particular, this implies that $\Ex[\ln(Z+1)] \geq \frac{1}{4}\ln(\alpha n + 1) > \frac{1}{4}\ln(\alpha n)$,
which finishes the proof.
\end{proof}

\cref{lem:lb} implies \cref{thm:informallowerbound} as follows:
set $k=d$, $k_{1} = \cdots = k_{d} = 1$ and $n \ge 100k \log(1/\delta)$.
Pick a distribution $D$ which is uniform over items, each satisfying exactly one of $d$ properties,
and with value drawn uniformly from $[0,1]$.

It suffices to show that with probability of at least $1/3$,
the algorithm retains an expected number of $\Omega(\log\log(1/\delta))$ items
from a constant fraction, say $1/4$, of the properties $i$.
This follows from \cref{lem:lb} as we argue next. Let $n_i$ denote the number of observed items of property $i$.
Then, since $\Ex[n_i] =n/d= n/k \geq 100$, the multiplicative Chernoff bound
implies that $n_i \geq n/2k \ge 2 \log(1/\delta)$ with high probability (probability $=1/2$ suffices).
Therefore, the expected number of properties $i$'s for which $n_i\geq 2 \log(1/\delta)$ is at least~$k/2$.
Now, consider the random variable $Y$ which counts for how many properties $i$ we have $n_i\geq 2 \log(1/\delta)$.
Since $Y$ is at most $k$ and $\Ex[Y] \ge k/2$, then
a simple averaging argument implies that with probability of at least $1/3$ we have that
$Y \ge k/4$.
Conditioning on this event (which happens with probability $\geq 1/3$), \cref{lem:lb}
 implies\footnote{Note that to apply \cref{lem:lb}  on $S_i$ we need $\delta > e^{-n_i/2}$, which is equivalent to $n_i > 2\ln(1/\delta)$.} that $\Ex[\lvert S_{i} \rvert] = \Omega(\log\log(1/\delta))$
for each of these $i$'s.

\section*{Acknowledgements}
We thank an anonymous reviewer for their remarks
regarding a previous version of this manuscript.
Their remarks and questions eventually led us to proving \Cref{thm:ramsey}.

\bibliographystyle{abbrvnat}
\bibliography{ref}

\appendix
\section{Deferred Proofs}

\subsection{Generalization and concentration}\label{sec:stabilityproofs}
\label{sec:generalization}

\paragraph{Technical notation.}
For $m\in \mathbb{N}$, the set $\{1,\ldots, m\}$ is denoted by $[m]$.
Given a family of sets $F$ over a domain $X$, and $Y\subseteq X$,
the family $\{f\cap Y: f\in F\}$ is denoted by $F|_{Y}$.
Recall that the VC dimension of $F$ is the maximum size of $Y\subseteq X$
such that $F|_Y$ contains all subsets of $Y$.

\begin{lemma*}[restatement of \cref{lem:vcupperbound}]
$\mbox{VC}({\cal R}) = O(d \log d)~.$
\end{lemma*}

\begin{proof}
Let $S$ be a set of items shattered by~${\cal R}$  and denote its size by $m$;
since $S$ is arbitrary, an upper bound on~$m$ implies an upper bound on $\mbox{VC}({\cal R})$.
To this end we upper bound the number of subsets in~${\cal R}|_S = \{S\cap R_{ T} : R_{ T}\in{\cal R}\}$.
Now, there are $m$ items in $S$ with at most $m$ different values. Therefore, we can restrict our attention to thresholds-policies where each threshold is picked from a fixed set of $m+1$ meaningful locations (one location in between values of two consecutive items when we sort the items by value).
Thus $\lvert {\cal R}|_S \rvert \le (m+1)^{d}$, but, as $S$ is shattered, $\lvert {\cal R}|_S \rvert = 2^{m}$ and we get $m \le d \log_{2} (m+1)$. This implies $m = O(d \log d)$ from which we conclude that~$\mbox{VC}({\cal R}) = O(d \log d)$.
\end{proof}
\begin{lemma*}[restatement of \cref{lem:vc2}]
$\mbox{VC}({\cal Q}) = O(\log d)~.$
\end{lemma*}
\begin{proof}
For $i\leq d$, let ${\cal Q}_i = \{R^{ T}_{i} :  T\in \T\}$.
Note that ${\cal Q} = \cup_i {\cal Q}_{i}$.
We claim that $\mbox{VC}({\cal Q}_{i}) = 1$ for all $i$.
Indeed, let $c',c''$ be two items.
Note that if $c'\notin P_i$ or $c''\notin P_i$ then $\{c',c''\}$ is not contained by ${\cal Q}_i$
and therefore not shattered by it.
Therefore, assume that $c',c''\in P_i$ and $v_{i}(c')\geq v_{i}(c'')$.
Now, it follows that any threshold $ T$ that retains $c''$ must also retain $c'$,
and so it follows that also in this case~$\{c',c''\}$ is not shattered.

The bound on the VC dimension of ${\cal Q} = \cup_{i\leq d}{\cal Q}_{i}$
follows from the next lemma.
\begin{lemma}\label{lem:union}
Let $m\geq 2$ and let $F_1,\ldots, F_m$ be classes with VC dimension at most $1$.
Then, the VC dimension of~$\cup_i F_i$ is at most~$10\log m$.
\end{lemma}
\begin{proof}
We show that $\cup_i F_i$ does not shatter a set of size $10\log m$.
Let $Y\subseteq X$ of size $10\log m$.
Indeed, by the Sauer's Lemma~\citep{Sauer72}:
\[\bigl\lvert (\cup_i F_i) |_Y \bigr\rvert  \leq m\left({10\log m \choose  0} + {10\log m \choose  1}  \right) = m(1+10\log m) < m^{10} = 2^{10\log m},\]
and therefore, $Y$ is not shattered by $\cup_i F_i$.
\end{proof}

This finishes the proof of \Cref{lem:vc2}. 
\end{proof}

\begin{lemma*}[restatement of \cref{lem:packing}]
There exists a collection of ${\cal N}$ thresholds-policies such that $\lvert {\cal N}\rvert \leq k^{O(d)}$,
and for every thresholds-policy $ T\in \T$ there are ${ T}^+,{ T}^-\in {\cal N}$ such that
\begin{enumerate}
\item  $V^{ T^-}(C)\leq V^{ T}(C)\leq V^{ T^+}(C)$ for every sample of items $C$.
(By taking expectations this also implies that $ \nu^{ T^-}\leq \nu^{ T}\leq  \nu^{ T^+}$.)
\item  $\nu^{ T^+} - \nu^{ T^-} \leq{10}$.
\end{enumerate}
\end{lemma*}

\begin{proof}
For every $i\leq d$ and $j\leq dn$ define thresholds $t_i^j\in [0,1]$ where $t_i^0=1$ and for $j>0$
set $t_i^j$ to satisfy\footnote{Such $t_i^j$'s exist due to our assumption that $D$ is atomless (see~\Cref{sec:thresholds}).}
\[\Pr_{c\sim D}[v(c) \geq t_i^j \text{ and } c\in P_i ] = \frac{j}{dn}.\]
Note that $t_i^0 > t_i^1 >\ldots$ (see \Cref{fig:tij}). 
Set
\[
{\cal J}_{i} = \biggl\{j : 0 \le \frac{j}{dn} \le \Pr_{c \sim D} [c \in P_{i}] ,j\in \mathbb{N}\biggr\},
\]
and define
\begin{align*}
{\cal N}_i &= \bigl\{t_{i}^{j} \, \mid \, j\in {\cal J}_{i} \cap \{0,1,\ldots,10dk\} \bigr\}\cup\{0\}\\
{\cal N} &= {\cal N}_1 \times {\cal N}_2 \ldots \times {\cal N}_d.
\end{align*}
Note that indeed $\lvert {\cal N}\rvert \leq (10dk +2) ^{{d+1}}  = k^{O(d)}$.

We next show that ${\cal N}$ satisfies items 1 and 2 in the statement of the lemma.
Let $ T\in \T$ be an arbitrary thresholds-policy.
The policies $ T^-=(t_i^-)_{i\leq d}$, and $ T^+ = (t_i^+)_{i\leq d}$
are derived by rounding $t$ in each coordinate
up and down respectively, to the closest policies in ${\cal N}$
(so, the thresholds in $T^+$ are smaller than in $T^-$;
the~``$+$'' sign reflects that it retains
more items and achieves a higher value).
Formally,
$t_i^+ = \max\{t\in {\cal N}_i : t \leq t_i\}$ and
$t_i^- = \min\{t\in {\cal N}_i : t \geq t_i\}$ where $t_i$ is the threshold for property $i$ in $T$.
%
Therefore, for every sample $ C\sim D^n$,
the set of items in $ C$ that are retained by~$ T$
contains the set retained by $ T^-$
and is contained in the set retained by $ T^+$.
This implies item 1.

To derive item 2, observe that for every sample $ C$:
$V^{ T^+}(C) - V^{ T^-}(C) \leq \lvert Z\rvert$,
where $Z\subseteq C$ denotes the set of items which participate in some {canonical} optimal solution for $ T^{+}$
that are not retained by $ T^{-}$.
Thus, it suffices to show that $\Ex[\lvert Z\rvert] \leq {10}$.
To this end put $p_i = \Pr_{c\sim D}[v(c) \geq t_i \text{ and } c\in P_i ]$ and partition $Z$ into two disjoint sets~$Z=E\cup F$, where
$E$ is the set of all items $ c_{j}\in Z$ that are assigned by the optimal solution of $ T^{+}$
to a property  $P_{i}$ where~$p_i < \frac{10k}{n}$, and
$F=Z\setminus E$.
We claim that
\begin{itemize}
\item $\Ex[\lvert E\rvert] \leq 1$:
for each
$P_i$ such that $p_i < \frac{10k}{n}$ let
$G_i\subseteq P_i$ denote the set of items whose value $v\in [t_i^+, t_i^-)$
(i.e.\ retained by $T^+$ and not by $T^-$ ). Note that $E\subseteq \cup_i G_i$,
and that $\Pr_{c\sim D}[c\in G_i] \leq \frac{1}{dn}$. Thus, it follows that
\[\Ex_{C\sim D^n}[\lvert E\rvert] \leq \Ex_{C\sim D^n}[\lvert \cup_i G_i\rvert]\leq \sum_i\Ex_{C\sim D^n}[\lvert G_i\rvert] \leq d\cdot\frac{n}{dn} \leq 1.\]
\item $\Ex[\lvert F\rvert ]\leq {9}$:
note that $\Ex[\lvert F\rvert ] \leq k\cdot\Pr[\lvert F\rvert>0]$ (because $F\subseteq Z$ and $\lvert Z\rvert \leq k$).
Thus, it suffices to show that~$\Pr[F>0]\leq \frac{9}{k}$.
Indeed,~$F\neq \emptyset$ only if there is a property $P_i$ with $p_i \geq \frac{10k}{n}$
such that less than~$k$ items {from $P_i$ are retained by $T^-$.
Fix a property $P_i$ such that $p_i \geq\frac{10k}{n}$ and let $p_i^- =  \Pr_{c\sim D}[v(c) \geq t_i^- \text{ and } c\in P_i ]$.
Since $p_i^-\geq \frac{10k}{n}$,
a multiplicative Chernoff bound yields that
\[ \Pr_{C\sim D^n}[\text{less than $k$ items from $P_i$ are retained by $T^-$}]\leq \exp\Bigl(-\frac{(9/10)^2}{2}10k\Bigr) \leq \frac{9}{k^2}
\leq\frac{9}{dk},\]}
and a union bound over all such properties $P_i$ implies that $\Pr[\lvert F\rvert >0]\leq \frac{9d}{dk}\le \frac{9}{k}$.
\end{itemize}
Thus, it follows that $v^{ T^+} - v^{ T^-} \leq 1 + k\cdot\frac{9}{k} = 10$, which finishes the proof.

\end{proof}

\begin{figure}
\includegraphics[scale=0.5]{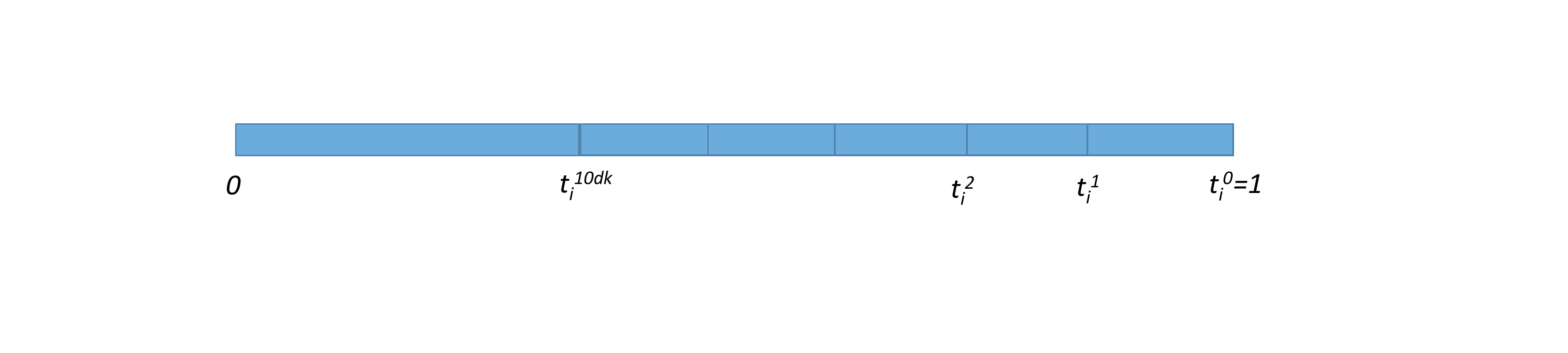}
\caption{An illustration of the thresholds in ${\cal N}_i$ as defined in the proof of \Cref{lem:packing}.
Each $t_i^j$ for $j\in {\cal J}_i$ satisfies $\Pr_{c \sim D}[v(c) \geq t_i^j \text{ and }c\in P_i ] = \frac{j}{dn}$.}\label{fig:tij}
\end{figure}

\end{document}